\documentclass{article}
\usepackage{amsmath,amsfonts,mathtools,graphicx}
\renewcommand{\baselinestretch}{1.1}
\usepackage{tikz}

\usetikzlibrary{arrows,positioning}
\usepackage{amsthm,mathrsfs}
\usepackage{amsfonts}
\usepackage[makeroom]{cancel}
\usepackage{float,placeins,afterpage}
\usepackage{etoolbox}
\usepackage{lscape}
\usepackage{breakcites}
\usepackage[shortlabels]{enumitem}
\usepackage{hyperref}
\usepackage{cleveref}
\usepackage{bm}
\usepackage{relsize}
\usepackage{epstopdf}
\usepackage{algorithm}
\usepackage[noend]{algpseudocode}

\usepackage{zref-base}
\usepackage{atveryend}
\makeatletter
\AfterLastShipout{%
	\if@filesw   
 		\begingroup
	     	\zref@setcurrent{default}{\the\value{equation}}%
     		\zref@wrapper@immediate{%
       		\zref@labelbyprops{eq-last}{default}%
		}%
   		\endgroup
    		\begingroup
      			\zref@setcurrent{default}{\the\value{thm}}%
      			\zref@wrapper@immediate{%
        			\zref@labelbyprops{thm-last}{default}%
   		   	}%
	    	\endgroup  
    		\begingroup
      			\zref@setcurrent{default}{\the\value{lem}}%
      			\zref@wrapper@immediate{%
        			\zref@labelbyprops{lem-last}{default}%
   		   	}%
	    	\endgroup  
    		\begingroup
      			\zref@setcurrent{default}{\the\value{defn}}%
      			\zref@wrapper@immediate{%
        			\zref@labelbyprops{defn-last}{default}%
   		   	}%
	    	\endgroup  
	\fi
}
\makeatother

\newtheorem{thm}{Theorem}
\newtheorem{lem}{Lemma}
\newtheorem{cor}{Corollary}
\newtheorem{defn}{Defnition}

\newtheorem{post}{Postulate}
\newtheorem{prop}{Proposition}

\crefname{defn}{definition}{Definition}
\crefname{post}{postulate}{Postulate}
\crefname{lem}{lemma}{lemmas}
\crefname{thm}{theorem}{theorems}
\Crefname{thm}{Theorem}{Theorems}
\crefname{cor}{corollary}{corollary}
\Crefname{cor}{Corollary}{Corollary}
\crefname{prop}{proposition}{Proposition}
\crefname{algorithm}{algorithm}{Algorithm}

\newcommand{\cov}{\mathbb{C}{\rm ov}}

\newcommand{\Cov}{\mathbb{C}{\rm ov}}

\newcommand{\E}{\mathbb{E}}

\newtoggle{final}
\toggletrue{final}

\iftoggle{final}{
\newcommand\domi[1]{{}}
\newcommand\michel[1]{{}}
\newcommand\naji[1]{{}}
}{
\newcommand\domi[1]{{\color{brown}Dominik: #1}}
\newcommand\michel[1]{{\color{red}Michel: #1}}
\newcommand\naji[1]{{\color{green}Naji: #1}}
}

\title{Telling cause from effect in deterministic linear dynamical systems\\ \small }

\author{Naji Shajarisales$^{\,\rm a}$, Dominik Janzing$^{\,\rm a}$, Bernhard Sch\"olkopf$^{\,\rm a}$ and Michel Besserve$^{\,\rm a,b}$\\
{\footnotesize {\em $^{\rm a}$ MPI for Intellingent Systems, T\"ubingen, Germany}}\\
{\footnotesize {\em $^{\rm b}$ MPI for Biological Cybernetics, T\"ubingen, Germany}}\\
{\footnotesize\texttt {\{{naji},
{dominik.janzing}, {bs}, {michel.besserve}\}@tuebingen.mpg.de}}}
\date{}
\begin{document}
\maketitle

\begin{quote}
{\small \hfill{\rule{13.3cm}{.1mm}\hskip2cm}
\textbf{Abstract}\vspace{1mm}

{\renewcommand{\baselinestretch}{1}
\parskip = 0 mm

Inferring a cause from its effect using observed time series data is a major challenge in natural and social sciences. Assuming the effect is generated by the cause trough a linear system, we propose a new approach based on the hypothesis that nature chooses the ``cause'' and the ``mechanism that generates the effect from the cause'' independent of each other. We therefore postulate that
 the power spectrum of the time series being the cause is uncorrelated with the square of the transfer function
of the linear filter generating the effect.  
While most causal discovery methods for time series mainly rely on the noise, our method relies on asymmetries of the power spectral density properties that can be exploited even in the context of deterministic systems. 
We describe mathematical assumptions in a deterministic model under which the causal direction is identifiable with this approach. We also discuss the method's performance under the additive noise model and
its relationship to Granger causality. Experiments show encouraging results on synthetic as well as real-world data. Overall, this suggests that the postulate of Independence of Cause and Mechanism is a promising principle for causal inference on empirical time series.}}


\vspace{-3mm}\hfill{\rule{13.3cm}{.1mm}\hskip2cm}
\end{quote}

\section{Introduction}
A major challenge in the study of complex natural systems is to infer the causal relationships between elementary characteristics of these systems. This provides key information to understand the underlying mechanisms at play and possibly allows to intervene on them to influence the overall behavior of the system. While causal knowledge is traditionally built by performing experiments, boiling down to modifying a carefully selected parameter of the system and analyzing the resulting changes, many natural systems do not allow such interventions without tremendous cost or complexity. For example, it is very difficult to influence the activity of a specific brain region without influencing other properties of the neural system \cite{logothetis2010}. Causal inference methods have been developed to avoid such intervention and infer the causal relationships from observational data only \cite{Spirtes1993,pearl2000causality}. To be able to build such knowledge without interventions, these approaches have to rely on key assumptions pertaining to the mechanisms generating the observed data.

The framework of causality in \cite{Spirtes1993,pearl2000causality} has originally addressed this question by modelling observations as i.i.d. random variables. However, observed data from complex natural system are often not i.i.d. and time dependent information reflects key aspects of those systems. Most causal inference methods for time series, including the most widely used Granger causality \cite{granger1969investigating}, assume the data is generated from a stochastic model through a structural equation linking past values to future ones through an i.i.d. additive noise term, the ``innovation of the process'' \cite{granger1969investigating,peters2012causal}. While these methods can successfully estimate the causal relationships when empirical data is generated according to the model assumptions, the results can be misleading when the model is misspecified. In particular, this is the case when unknown time lags are introduced in the measured time series.

In this paper, we introduce a new approach to inferring causal directions in time series, the Spectral Independence Criterion (SIC). The idea behind SIC, as well as several new approaches to causal inference \cite{daniusis2010inferring,janzing2012information,JanzingHS2010,Zscheischler2011}, is to rely on the `philosophical' principle that  the cause and the mechanism that generates the cause from the effect are chosen independently by  Nature. Thus, these two objects should
not contain any information about each other \cite{janzing2010causal,LemeireJ2012,Schoelkopf2012}. 
Here, we refer to this abstract principle as the postulate of Independence of Cause and Mechanism (ICM). 
The above mentioned methods relying on ICM refer to different domains and rely on quite different formalizations of the concept of ``independence''. 
 SIC formalizes the  ICM postulate in the context where both   cause and effect are stationary time series and  
  the cause generates the effect trough a linear time invariant filter. 
The SIC postulate assumes that the frequency spectrum of the cause does not correlate with the transfer function of the filter. This assumption is justified by its connection to the Trace Method \cite{JanzingHS2010} and by a generative model of the system.
Under this postulate, we prove that SIC can tell the causal direction of the system from its anti-causal counterpart.
Moreover, we elaborate on the connection between this novel framework and linear Granger causality, showing they are exploiting fundamentally different information from the observed data. In addition, superiority to Granger causality is shown analytically in the context of time series measurements perturbed by an unknown time lag. We perform extensive experimental comparisons, both on simulated and real datasets. In particular, we show that our approach outperforms Granger causality to estimate the direction of causation between to structures of rat hippocampus using Local Field Potential (LFP) recordings.

Overall, the proposed method offers a new approach to causal inference for time series data with identifiability results, and shows unprecedented robustness 
to measurement delays. The promising empirical results suggest the SIC postulate is a reasonable assumption for empirical data, and that it should be further exploited to develop novel causal inference techniques. 

\section{Spectral Independence Criterion (SIC)}
\subsection{Notations and model description}
\label{subsec:math}
We refer to a sequence of real or complex numbers $\textbf{a}=\{a_t , t\in \mathbb{Z}\}$ as a {\it deterministic time series}. Its discrete Fourier transform is defined by
$$
\widehat{a}(\nu)=\sum_{t\in \mathbb{Z}} a_t \exp(-\mathbf{i}2\pi\nu t),\, \nu \in \left[-1/2,\,1/2\right]=:\mathcal{I}
$$
The {\it energy} of the deterministic time series is the squared $l^2$ norm: $\|\textbf{a}\|_{2}^2=\sum_t |a_t|^2$. For ease of notation we will also use the Z-transform of $\textbf{a}$
$$
\widetilde{a}(z)=\sum_{t\in \mathbb{Z}} a_t z^{-t},\, z\in \mathbb{C}
$$
such that $\widehat{a}(\nu)=\widetilde{a}(\exp(\mathbf{i}2\pi\nu))$.

 We assume that the causal mechanism is given by 
a (deterministic) Linear Time Invariant (LTI) filter.
That is, the causal mechanism is formalized by the convolution
\begin{equation}\label{eq:conv}
\textbf{y}=\{y_t\}=\{\sum_{\tau\in \mathbb{Z}} x_{t-\tau} h_\tau\,\}=\textbf{x}*\textbf{h}, 
\end{equation}
where $\textbf{h}$ denotes the {\it impulse response}, $\textbf{x}$ the input time series and $\textbf{y}$ the output. We will assume that the filter satisfies the Bounded Input Bounded Output (BIBO) stability property \cite{proakis2001digital}, which boils down to the condition $\|\textbf{h}\|_1<+\infty$. Under this assumption, the Fourier transform $\widehat{h}$ is well defined and we call it the transfer function of the system.

We assume that the input time series $\textbf{x}$ is a sample drawn from a \textit{stochastic process}, $\{X_t,t\in \mathbb{Z}\}$. For a given index $t$, $X_t$ represents the random variable at index $t$. We use $\{X_t\}$ or simply $\textbf{X}$ to represent the complete stochastic process. We use $X_{t:s}$ to indicate the random vector corresponding to the restriction of the time series to the integer interval $[t \,..\, s]$. We use $X_{t:s}$ to indicate the random vector corresponding to the restriction of the time series to the integer interval $[t \,..\, s]$.
 Assuming $\textbf{X}$ is a zero mean stationary process (in this paper, stationary will always stand for weakly or wide-sense stationary  \cite{brockwell2009time}), we will denote by $ C_{xx}(\tau) = \mathbb{ E} [ X_t X_{t+\tau}] $  the autocovariance function of the process and assume it is absolutely summable. Then, we can define its {\it Power Spectral Density} (PSD) $S_{xx}=\widehat{C_{xx}}$. Under these assumptions, the power of the process $P(\textbf{X})=\mathbb{E}(|X_t|^2)$ is finite and $P(\textbf{X})=\int_{-1/2}^{1/2} S_{xx}(\nu)d\nu$, such that $S_{xx}$ belongs to $L^1$. Moreover, we recall the following basic properties for our model:
\begin{prop}
Assume the weakly stationary input $\textbf{X}$ is filtered by the BIBO linear system of impulse response $\textbf{h}$ to provide the output $\textbf{Y}$. Then $\|\textbf{h}\|_{2}^2<+\infty$, $\widehat{h}\in L^\infty$ and $\textbf{Y}$ is weakly stationary with summable autocovariance such that
\begin{equation}
S_{yy}(\nu)=|\widehat{h}(\nu)|^2 S_{xx}(\nu), \nu\in \mathcal{I}\label{eq:ampli}\end{equation}
\end{prop}
\begin{proof}
Results from elementary properties of the Fourier transform and Proposition 3.1.2. in \cite{brockwell2009time}.
\end{proof}
 
If such a linear filtering relationship exists for $\textbf{X}$ as input and $\textbf{Y}$ as output, but not in the opposite way, we can use this information to infer that $\textbf{X}$ is causing $\textbf{Y}$ and not the other way round. If there are such impulse responses exist for both directions, say $h_{\textbf{X}\to \textbf{Y}}$ and $h_{\textbf{Y}\to \textbf{X}}$,
their Fourier transforms are related by 
\[
\widehat{h}_{\textbf{X}\to \textbf{Y}}=\frac{1}{\widehat{h}_{\textbf{Y}\to \textbf{X}}}\,,
\]
and we have to resort to a more refined criterion for the causal inference. We will assume this situation in the remaining of the paper.

\subsection{Definition of SIC}
Assume we are given the two processes $\textbf{X}:=\{X_t,t\in Z\}$ and $\textbf{Y}:=\{Y_t,t\in Z\}$.
Moreover, we assume that 
 exactly one
of the following two alternatives is true:
(1)  $\textbf{X}$ causes $\textbf{Y}$   or (2)   $\textbf{Y}$ causes $\textbf{X}$. We assume that there are no unobserved common causes of $\textbf{X}$ and $\textbf{Y}$. Our causal inference problem thus reduces to a binary decision. 
In the spirit of ICM, we assume that in case (1), $\textbf{X}$ and $\textbf{h}$ should not contain information about each other and our  Spectral Indpendance Criterion (SIC) assumes that the input power does not correlate with
the amplifying factor, that is,
\begin{equation}\label{eq:sic}
\langle S_{xx} |\widehat{h}|^2 \rangle =\langle S_{xx}\rangle  \langle |\widehat{h}|^2\rangle \,,
\end{equation}
where $\langle f \rangle=\int_\mathcal{I}f(\nu)d\nu$ denotes the average over the unit frequency interval $\mathcal{I}$.
Note that the left hand side of (\cref{eq:sic}) is the average intensity of the output signal
 $\{Y_t,t\in Z\}$ over all frequencies. Hence, SIC states that the average output intensity is the
 same as amplifying all frequencies by the average amplifying factor.  To motivate why we call (\cref{eq:sic})  an {\it independence} condition we note that
the difference between the left and the right hand side can be written as a covariance:
\begin{align*}
\langle S_{xx} \cdot |\widehat{h}|^2 \rangle -\langle S_{xx}\rangle \langle |\widehat{h}|^2\rangle=
 \mathbb{C}{\rm ov}\left( S_{xx}, |\widehat{h}|^2 \right)\,.   
\end{align*}
were we consider $S_{xx}$ and $|\widehat{h}|^2$ as functions of the random variable $\nu$ uniformly distributed on $\mathcal{I}$. As a consequence statistical independence between those random variables implies that (\cref{eq:sic}) is satisfied. 

Note that the criterion (\cref{eq:sic}) can be rephrased in terms of the power spectra of $\textbf{X}$ and $\textbf{Y}$ alone using
(\cref{eq:ampli}), which are closer to observable quantities than $\widehat{h}$:
\begin{post}[Spectral Independence Criterion]
If $\textbf{Y}$ is generated from $\textbf{X}$ by a linear deterministic translation invariant system then we have:
\begin{gather}\label{eq:sicob}
\langle S_{yy} \rangle =\langle S_{xx}\rangle  \langle S_{yy}/S_{xx} \rangle\,.
\end{gather}
\end{post}

\subsection{Quantifying violation of SIC}
This motivates us to define a measure of dependence between the input PSD on one hand and transfer function of the mechanism on the other hand. To asses to what degree such a relation holds we introduce a scale invariant expression $\rho_{\textbf{X}\to \textbf{Y}}$, that we call the spectral dependency ratio (SDR) from $\textbf{X}$ to $\textbf{Y}$:
\begin{equation}
\label{eq:sdr}
\rho_{\textbf{X}\to \textbf{Y}}:=\frac{\langle S_{yy}\rangle  }{\langle S_{xx}\rangle \langle S_{yy}/S_{xx}\rangle }
\end{equation}
Here, the value $1$ means independence, which becomes more obvious by rewriting (\cref{eq:sdr}) as
\[
\rho_{\textbf{X}\to \textbf{Y}} = \frac{\Cov[S_{xx},|\widehat{h}|^2]}{ \langle S_{xx}\rangle \langle S_{yy}/S_{xx}\rangle } +1\,. 
\]
Finally, we note that $\rho_{\textbf{X}\to \textbf{Y}}$ can be written in terms of total power and energy:
\begin{gather*}
\rho_{\textbf{X}\to \textbf{Y}}=\frac{P(\textbf{Y})}{P(\textbf{X})||\textbf{h}||_{2}^2}
\end{gather*}
We then define $\rho_{\textbf{Y}\to \textbf{X}}$ by exchanging the roles of $\textbf{X}$ and $\textbf{Y}$:
\begin{gather}\label{eq:delta_infty2}
\rho_{\textbf{Y}\to \textbf{X}}:=
\frac{\langle S_{xx}\rangle  }{\langle S_{yy}\rangle \langle S_{xx}/S_{yy}\rangle }
\end{gather}

\subsection{Identifiability results}
In order to identify the true causal direction from SIC, it is necessary to show that $\rho_{\textbf{X}\to \textbf{Y}}$ and $\rho_{\textbf{Y}\to \textbf{X}}$ take characteristic values that are informative about this inference problem. The following first result shows explicitly how dependence measures in both directions are related:
\begin{prop}{\bf(Forward-backward inequality)}\label{lem:violation}
For a given linear filter with input PSD $S_{xx}$, output PSD $S_{yy}$ and a non-constant modulus transfer function  $\widehat{h}$ we have
\begin{equation}
\label{eq:sicviol}
\rho_{\textbf{X}\to \textbf{Y}}.\rho_{\textbf{Y}\to \textbf{X}}<1\,.
\end{equation}
Moreover, if $\,\exists \alpha>0, \forall \nu \in \mathcal{I},|\widehat{h}(\nu)|^2\leq (2-\alpha) \|\textbf{h}\|_2^2\,$, 

then
\begin{gather}\label{eq:sicviol2}
\rho_{\textbf{X}\to \textbf{Y}}.\rho_{\textbf{Y}\to \textbf{X}}\leq \left[1+\alpha \int_\mathcal{I} \left(\frac{|\widehat{h}(\nu)|^2-\|\textbf{h}\|_2^2}{\|\textbf{h}\|_2^2}\right)^2\!\! d\nu \right]^{-1}\!\!\!\!<1\,.
\end{gather}
\end{prop}
Proof of this proposition is given in \textit{supplementary material}. Note that $\|\textbf{h}\|_2^2$ corresponds to the mean value of the transfer function due to Parseval's theorem. According to equation \cref{eq:sicviol2}, the less constant $|\widehat{h}|^2$ is, the more the product of the independence measures will be inferior to $1$. Assuming the SIC postulate is satisfied in the forward direction such that $\rho_{\textbf{X}\to \textbf{Y}}=1$, it follows naturally that $\rho_{\textbf{Y}\to \textbf{X}}<1$. The two causal directions can thus be distinguished well whenever the transfer function deviates significantly from its mean value such that $\rho_{\textbf{X}\to \textbf{Y}}\rho_{\textbf{Y}\to \textbf{X}}$ is bounded away from 1. We then infer the causal direction to be the one with the largest $\rho$ value.

To further support that SDR values can be used empirically for causal inference, we need the SIC postulate to be approximately satisfied (see (\cref{eq:sicob})) in systems generated according to the ICM principle. We now describe a model where $\textbf{h}$ is generated by some random process, independently of $\textbf{X}$. To this end, assume we start with a Finite Impulse Response (FIR) $\textbf{h}$, that is, $h_\tau=0$ for all $\tau \geq m$, for some $m$. Then $\textbf{h}$ is given by $m$ real numbers $b_1,\dots,b_m$ such that 
\[
h_i= b_{i} \quad i=0,\dots,m-1\,.
\]
We then apply an orthogonal transformation $\textbf{U}$, randomly drawn from 
the orthogonal group $O(m)$
according to the `uniform distribution' on $O(m)$,
that is, the Haar measure. 
In this way, we generate a new impulse response function
\begin{equation}\label{eq:hprime}
h'_i :=(\textbf{U}\textbf{b})_{i}   \quad i=0,\dots,m-1\,.
\end{equation}
Since orthogonal transformations preserve the Euclidean norm by definition, they preserve the energy of the filter. Our procedure thus chooses a random filter among the set of filters having the same support of length $m$
and the same energy. We now show that for large $m$ the resulting filter will approximately satisfy SIC 
with high probability:
\begin{thm}{\bf (concentration of measure for FIR filters)}\label{thm_COM_FIR} 
For some fixed $S_{xx}$, let $\rho^U_{\textbf{X}\to\textbf{Y}}$ be the dependence measure 
obtained for $h'$ in (\cref{eq:hprime}). If $U$ is chosen from the Haar measure on $O(m)$, then for any given $\varepsilon$
\begin{gather*}
|\rho^U_{\textbf{X}\to\textbf{Y}}- 1|\leq \frac{2 \varepsilon}{P(\textbf{X})} \max\limits_{\nu} S_{xx}(\nu)\,.
\end{gather*}
with probability $\delta:=1-\exp(\kappa(m-1)\varepsilon^2)$ where $\kappa$ is a positive global constant independent of $m$, $\varepsilon$, $\textup{\textbf{X}}$ and $\textup{\textbf{Y}}$.
\end{thm}
Proof of this theorem is provided in \textit{supplementary material}. This result provides a justification for using SIC provided that the dimension of the vector of filter coefficients $m$ is large enough. The relevance of $m$ will be investigated in practice in the experimental section.

\subsection{Relation to the Trace Condition}
We now describe the relation between SIC and a causal inference tool called Trace Method
\cite{JanzingHS2010}. Let $X$ and $Y$ be $n$-dimensional variables, related by the linear structural equation 
\[
Y=AX+E\,,
\]
where $A$ is an $m\times n$ structure matrix and $E$ is a $n$-dimensional noise variable independent of $X$.
\cite{JanzingHS2010} postulate the following independence condition between the covariance matrix of input distribution $\Sigma_X$ and $A$:
\begin{post}[Trace Condition]\label{pos:tc} 
\begin{equation}
\tau_m(A\Sigma_X A^T) = \tau_n (\Sigma_X) \tau_n (A^TA) \,,
\label{eq_tc}
\end{equation}
approximately, where $\tau_n(B)$ denotes the renormalized trace ${\rm tr}(B)/n$.
\end{post}
The postulate can be justified by random matrix theory with large $m$ when $A$ and $\Sigma_X$  are independently chosen according to priors satisfying appropriate symmetry assumptions \cite{JanzingHS2010}. In the association between SIC and trace method we only consider square matrices and therefore $m=n$.

To quantify the violation of (\cref{eq_tc}) we introduce the following quantity:
\begin{defn}[Tracial Dependency Ratio (TDR)]
The tracial dependency ratio is given by
\begin{equation}\label{eq:tdr}
r_{X\to Y}:= \frac{ \tau_n(A\Sigma_X A^T) }{ \tau_n (\Sigma_X) \tau_n (A^TA) }\,.
\end{equation}
\end{defn}
We thus can see that the tracial ratio plays a role analog to our spectral dependency ratio $\rho$ in the finite dimensional case. We can actually show that SIC can be viewed as a limit case of the Trace Condition by defining the following truncated system.
\begin{defn}\label{def:trunc}
To any given infinite dimensional linear system $\textbf{X}\mapsto \textbf{Y}=\textbf{h}*\textbf{X}$, the truncated system of order $N$ is defined by zeroing the input and the output values for integers $k$ such that $-N\leq k < N$:
$$\textbf{X}'_{N}=\textbf{X}_{-N:N-1} \mapsto \textbf{Y}'_{N}=(\textbf{h}*\textbf{X}'_N)_{-N:N-1},$$
\end{defn}
Note that in this definition for each $N$, the vectors $\textbf{Y}'_{-N:N-1}$ are inherently different. The mapping defined in this way is linear and can be written as $\textbf{Y}'=\textbf{H}\textbf{X}'$ with $[H]_{ij}=h_{i-j}$, such that the trace method can be applied to it. We then have the following result showing that SIC can be obtained from the Trace Condition as an appropriate
limit:
\begin{thm}\label{thm_limit}
Let $r_{\textbf{X}'_N\to \textbf{Y}'_N}$ represent the tracial ratio for the truncated systems of order $N$ for a given linear system with SDR $\rho_{\textup{\textbf{X}}\to\textup{\textbf{Y}}}$. Then
\[
\lim_{N\to\infty} r_{\textbf{X}'_N\to \textbf{Y}'_N} =\rho_{\textup{\textbf{X}}\to\textup{\textbf{Y}}}
\]
\end{thm}
The proof, together with two necessary lemmas is available in \textit{supplementary material}.

\section{SIC for vector autoregressive models}
SIC and Granger causality rely on completely different assumptions
but both apply to linear time series models. In this section, we study
the classical Vector Autoregressive (VAR) model used in Granger causality
from the SIC perspective to better understand the relation.

\subsection{VAR model}
We assume the observed time series are generated by a VAR
model such that $x$ Granger causes $y$.
\begin{eqnarray}
X_t & = & \sum_k a_{k}X_{t-k}+\epsilon_t\label{eq:1a}\\
Y_t & = & \sum_k b_{k}Y_{t-k}+\sum_k c_{k}X_{t-k}+\xi_{t}\label{eq:1b}
\end{eqnarray}
Both noise terms $\epsilon$ and $\xi$ in this expression are i.i.d
normal noises. 

\subsection{Applying SIC to VAR models}
We want to rewrite this expression such that $\textbf{Y}$ is obtained from $\textbf{X}$
by a deterministic linear time invariant filter. We observe that the VAR model can be cast as linear time invariant filter
 if we neglect the additive noise
$\xi$. Indeed, then the mechanism is the following ARX (AutoRegressive with eXogenous input) model \cite{keesman2011system}.
\begin{equation}
Y_t=\sum_k b_{k}Y_{t-k}+\sum_k c_{k}X_{t-k}\label{eq:2}
\end{equation}
Using basic properties of the Z-transform, we can
derive the following analytic expressions of the input PSD $S_{xx}$: \domi{ maybe derive it in the appendix?}
\[
S_{xx}(\nu)=|\widehat{n}(\nu)|^{2}=|\widetilde{n}(\exp(2\pi\mathbf{i}\nu))|^{2}\,,
\]
with
\[
\widetilde{n}(z)=\frac{1}{1-\sum_{k} a_{k}z^{-k}}.
\]
Moreover, the transfer function corresponding to the mechanism
in equation~\cref{eq:2} is
\[
\widetilde{m}(z)=\frac{\sum_{k} c_{k}z^{-k}}{1-\sum_{k} b_{k}z^{-k}}
\]
As a consequence, testing SIC on the VAR model in the forward direction amounts (when neglecting
the filtered noise $\xi$), to test independence between
\begin{equation}\label{eq:trans}
|\widehat{h}(\nu)|^2=|\widetilde{m}(\exp(2\pi\mathbf{i}\nu))|^{2}
\end{equation}
 and
\begin{equation}\label{eq:input}
S_{xx}(\nu)=|\tilde{n}(\exp(2\pi \mathbf{i} \nu)|^2\,,
\end{equation}
 which are parametrized by the coefficients $\{b_{k},c_{k}\}$
and $\{a_{k}\}$ respectively.
We conjecture that
a concentration of measure result similar to \Cref{thm_COM_FIR} holds
stating that independent choice of the coefficients from an
appropriate symmetric distribution typically yields small correlations between (\cref{eq:trans}) and (\cref{eq:input}). This will be tested empirically in the Experiments section. Additionally, the robustness of our approach to noise in the VAR model will be addressed extensively in a longer version of this manuscript.

\michel{ONLY FOR PAPER VERSION BUT WE SHOULD MENTION IT SOMEWHERE... LIKE "THE EFFECT OF NOISE WILL BE ADDRESSED IN A LONGER VERSION?" In practice, we reduce the influence of the neglected noise term $\xi$ by fitting
a linear predictive model to equation~(\cref{eq:1b}). The influence of the noise term $\xi$ affecting only the residual error $R$ of
the prediction of $\textbf{Y}$ given $\textbf{X}$, we apply the SIC criterion to the denoised output $Y-R$. Check there can't be an estimation bias induced by the noise term in empirical time series}

\subsection{Comparison of SIC and Granger causality}
The bivariate VAR model above is the typical model where Granger causality works. To recall the idea of the latter, note that
it infers that there is an influence from $\textbf{X}$ to $\textbf{Y}$ whenever
predicting $\textbf{Y}$ from its past is improved by accounting for the past of $\textbf{X}$. Rephrasing this in terms of conditional independences, $\textbf{X}$ is inferred to cause $\textbf{Y}$ whenever 
$Y_t$ is not conditionally independent of $X_{t-1},X_{t-2},\dots$, given $Y_{t-1},Y_{t-2},\dots$. Within the context of the above linear model, knowing $X_{t-1},X_{t-2},\dots$ reduces the variance of $Y_t$, given $Y_{t-1},Y_{t-2},\dots$ because then
the noise $\nu_t$ is the only remaining source of uncertainty.
Without knowing $X_{t-1},X_{t-2},\dots$, we have additional uncertainty due to the contribution of $\epsilon_{t-1},\epsilon_{t-2},\dots$. 

SIC, on the other hand, does not rely on detecting whether
$\textbf{X}$ helps in improving the prediction of $\textbf{Y}$. 
As demonstrated above, SIC applied to a bivariate VAR model
boils down to quantifying independence between two linear filters
defined by set of coefficients, the filter generating the input
with transfer function $\widehat{n}$ and the filter of the mechanism with
transfer function $\widehat{m}$. 
This is a completely different concept. One can easily imagine that the coefficients 
 $\{b_{k},c_{k}\}$
and $\{a_{k}\}$ can be hand-designed such that the functions
(\cref{eq:trans}) and (\cref{eq:input}) are correlated. 
This would spoil SIC, but leave Granger unaffected.
On the other hand, the following subsection describes a scenario where Granger fails but SIC still works.

\subsection{Sensitivity to Time Lag}
Consider two time series $\{X_t\}$ and $\{Y_t\}$ where $\{X_t\}$ is a white noise and
\begin{gather*}
\forall t\in\mathbb{Z},\quad \indent Y_t=c Y_{t-1}+X_{t-1},
\end{gather*}
for a given $c$. It can be easily seen that this type of input and output can be simulated using an IIR filter with $(a_1,a_2)=(1,c)$ and $b_1=1$ in (\cref{eq:diff}) and the rest of the coefficients are zero (please refer to the definition of coefficients in section \cref{sec:synthetic}). The infinite DAG for this causal structure can be seen in   \cref{fig:dag}.
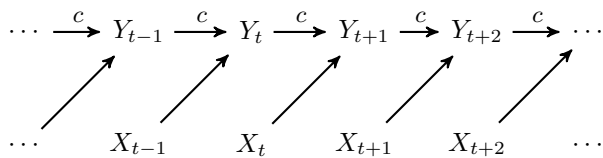
\begin{figure}[!hbt]
\centering
\begin{tikzpicture}[scale=0.15,->,>=stealth',shorten >=1pt,auto,node distance=1.5cm,
  thick,main node/.style={font=\sffamily\bfseries}]
  
  \node[main node] (0) {$\cdots$};
  \node[main node] (1) [right of=0] {$Y_{t-1}$};
  \node[main node] (2) [right of=1] {$Y_t$};
  \node[main node] (3) [right of=2] {$Y_{t+1}$};
  \node[main node] (4) [right of=3] {$Y_{t+2}$};
  \node[main node] (inf)[right of=4] {$\cdots$};

  \node[main node] (0')[below of=0] {$\cdots$};
  \node[main node] (5) [below of=1]{$X_{t-1}$};
  \node[main node] (6) [below of=2] {$X_t$};
  \node[main node] (7) [below of=3] {$X_{t+1}$};
  \node[main node] (8) [below of=4] {$X_{t+2}$};
  \node[main node] (inf')[below of=inf] {$\cdots$};

  \path[every node/.style={font=\sffamily\small}]
    (0) edge node[above] {$c$} (1)
    (1) edge node[above] {$c$} (2)
    (2) edge node[above] {$c$} (3)
    (3) edge node[above] {$c$} (4)
    (4) edge node[above] {$c$} (inf)

    (0') edge (1)
    (5) edge (2)
    (6) edge (3)
    (7) edge (4)
    (8) edge (inf);

\end{tikzpicture}
\caption{The original causal structure with instantaneous causal effect}
\label{fig:dag}
\end{figure}

Now if there would be a measurement delay of length $k$ for \textbf{Y}, the observed values will be a new time series, say $\tilde{\textup{\textbf{Y}}}$, where $\tilde{Y}_t=Y_{t-k}$. Although the ground truth is $\textup{\textbf{X}}\to\tilde{\textup{\textbf{Y}}}$ independent of $k$, Granger causality only infers the correct causal structure if $k\leq 0$ (where there is a lag in measurement of $\textbf{X}$, but not $\textbf{Y}$). However SIC always infers the correct direction (except when $c=0$ and the time structure is spoiled). This is because the PSD of the white noise $\textup{\textbf{X}}$ is constant and depends only on the total power, i.e,
\[
S_{xx}(\nu) = {\rm Var} (X_t)=P\{X\}\,,
\]
for all $\nu \in [-1/2,1/2]$.
and obviously, this constant remains the same for the lagged time series. Thus, SIC correctly identifies the causal structure (except when $c=0$ in which case the dependence to time is completely spoiled).\domi{should we discuss the irrelevance of the lag only for white noise? Is it really difficult for the general case?}

\section{Experiments}
In this section we study our causal inference algorithm using synthetic experiments and apply it to several real world data sets.
\subsection{Synthetic Data: ARMA filters and processes}
\label{sec:synthetic}
We designed synthetic experiments to assess the validity of the SIC approach. The data generating process is as follows. 
The LTI system $\mathcal{S}$ modeling the mechanism is chosen among the family of \textit{ARMA($FO$,$BO$) filters} with parameters $(\textbf{a},\textbf{b})$ defined by input-output difference equation:
\begin{gather}
\label{eq:diff}
y_n  = \frac{1}{a_{0}}(\sum_{i=0}^{FO} b_{i} x_{n-i} + \sum_{j=1}^{BO} a_{j} y_{n-j}).
\end{gather}
 For these filters $FO$ is known as the feedforward order and $BO$ is the feedback order. $a_i$'s and $b_i$'s are known as feedback and feedforward coefficients respectively. Note that when $FO(\mathcal{S})=0$, the system is called and autoregressive filter. Alternatively, $BO(\mathcal{S})=0$ corresponds to the family of Finite Impulse Response (FIR) or Moving Average filters. Whenever $BO(\mathcal{S})\neq0$, the filter has Infinite Impulse Response (IIR). The input of the causal model will be chosen among the family of \textit{ARMA($FO$,$BO$) processes}, which are generated by filtering an i.i.d noise input with an ARMA($FO$,$BO$) filter. We thus chose two filters $\mathcal{S}$ and $\mathcal{S}'$, with parameters $(\textbf{a},\textbf{b})$ and $(\textbf{a}',\textbf{b}')$ respectively.
To simulate a cause effect pair \textbf{X},\textbf{Y}, we generated the cause \textbf{X} by applying $\mathcal{S}$  to a normally distributed i.i.d noise. Then,
 we generated \textbf{Y} by applying  $\mathcal{S}'$ to
 $\textbf{X}$. The feedforward and feedback orders of both systems $\mathcal{S}$ and $\mathcal{S}'$ were chosen identical in all experiments.

In each trial all the elements of vectors $\textbf{a}$, $\textbf{a}'$, $\textbf{b}$ and $\textbf{b}'$ except the first ones (i.e. $a_0,b_0,a'_0,b'_0$ which were fixed to one) were sampled from an isotropic multidimensional Gaussian distribution with variance $0.01$. Coefficients are sampled using rejection sampling such that only BIBO-stable filters are kept.
\domi{I would suggest to also try non-isotropic priors, e.g., with decay.
It's more interesting whether the method also works when the generating process does not coincide with the `ideal' one} \naji{I have tried it but I am not sure how to report it. It actually works perfectly fine similar to the isotropic case. What I did was sorting the coefficients after sampling them isotropically)}

 We simulated sequences of length $10000$. The PSD of $\textbf{X}$ and $\textbf{Y}$ were estimated using Welch's method \cite{welch1967use}. We repeated this experiment $1000$ times. Figure~\cref{fig:Delta_Hist-BO=FO=20_nonoise} shows an example of the distribution of $\rho_{\textbf{X}\to \textbf{Y}}$ and $\rho_{\textbf{Y}\to \textbf{X}}$ and of their difference using 
$FO(\mathcal{S})=BO(\mathcal{S})=FO(\mathcal{S}')=BO(\mathcal{S}')=5$.

\begin{figure}[!hbt]
\centerline
{\includegraphics[width=\linewidth]{\detokenize{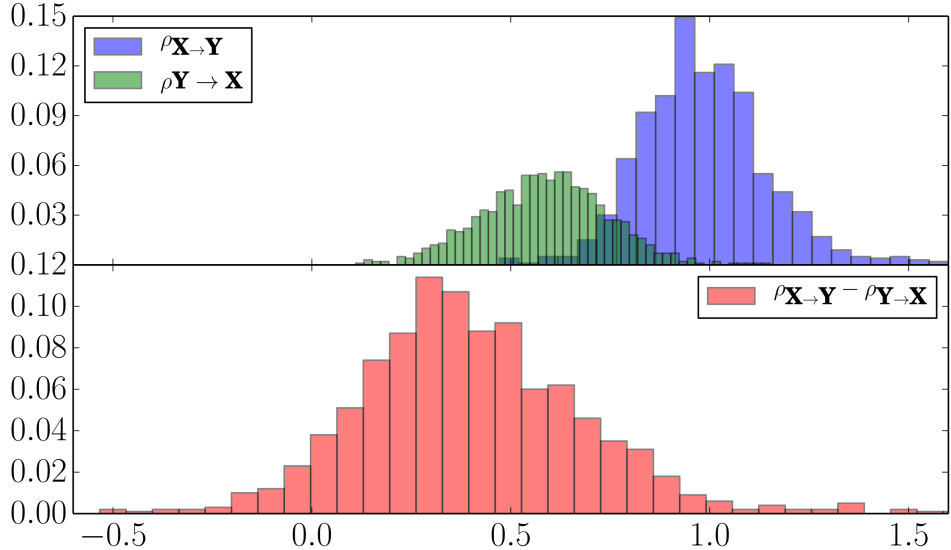}}}
  \caption{Top plot: Histogram for the estimators of $\rho_{X\to Y}$ and $\rho_{Y\to X}$. Bottom plot: Histogram of the estimated difference $\rho_{X\to Y}-\rho_{Y\to X}$}
  \label{fig:Delta_Hist-BO=FO=20_nonoise}
\end{figure}

The SDR for the correct direction of is concentrated around one, while in the wrong direction the estimator stays less inferior to one for most of its probability mass (in this example $\%97.3$). This results in a positive difference between SDR for most of the probability mass. Accordingly, our inference algorithm based on the sign of this  difference \cref{alg:SIC} will select the correct direction in most of the cases. 

\begin{algorithm}
\caption{SIC\_Inference}\label{alg:SIC}
\begin{algorithmic}[1]
\Procedure{SIC\_Inference(\textbf{X},\textbf{Y})}{}
\State $S_{xx} \gets \text{spectrum of } \textbf{X}$
\State $S_{yy} \gets \text{spectrum of } \textbf{Y}$
\State Calculate $\rho_{\textbf{X}\to \textbf{Y}}$ and $\rho_{\textbf{Y}\to\textbf{X}}$ using (\cref{eq:sdr})
\State \emph{Inference Step}:
\If {$\rho_{\textbf{X}\to\textbf{Y}}>\rho_{\textbf{Y}\to \textbf{X}}$} 
\State return $\textbf{X}\to \textbf{Y}$
\Else   
\State return $\textbf{Y}\to \textbf{X}$
\EndIf
\EndProcedure
\end{algorithmic}
\end{algorithm}
Based on this inference algorithm, we test the effect of the filter orders on the performance of the method, where we evaluate the performance of each setting of $FO(\mathcal{S'})$ and $BO(\mathcal{S'})$ over 1000 trials. We varied the orders between $2$ and $21$ and compared the performance of the cases $FO(\mathcal{S'})=BO(\mathcal{S'})$, $FO(\mathcal{S'})=0$ and $BO(\mathcal{S'})=0$. Considering that the experiments are independent and based on the assumption that our method is successful with probability $p$ where $p$ has a binomial distribution, we calculated confidence intervals using Wilson's score interval \cite{wilson1927probable} where $\alpha=0.05$ (and therefore $z_{\alpha/2}=1.96$). The performance increases rapidly with filter order, as can be seen in the plots of ~\cref{fig:perf-dim-fo}. Moreover, the feedforward filter coefficients seem the most beneficial to the approach, since their absence leads to the worst performance (~\cref{fig:perf-dim-fo} red line).

\begin{figure}[!hbt]
\centerline{\includegraphics[width=1\linewidth]
{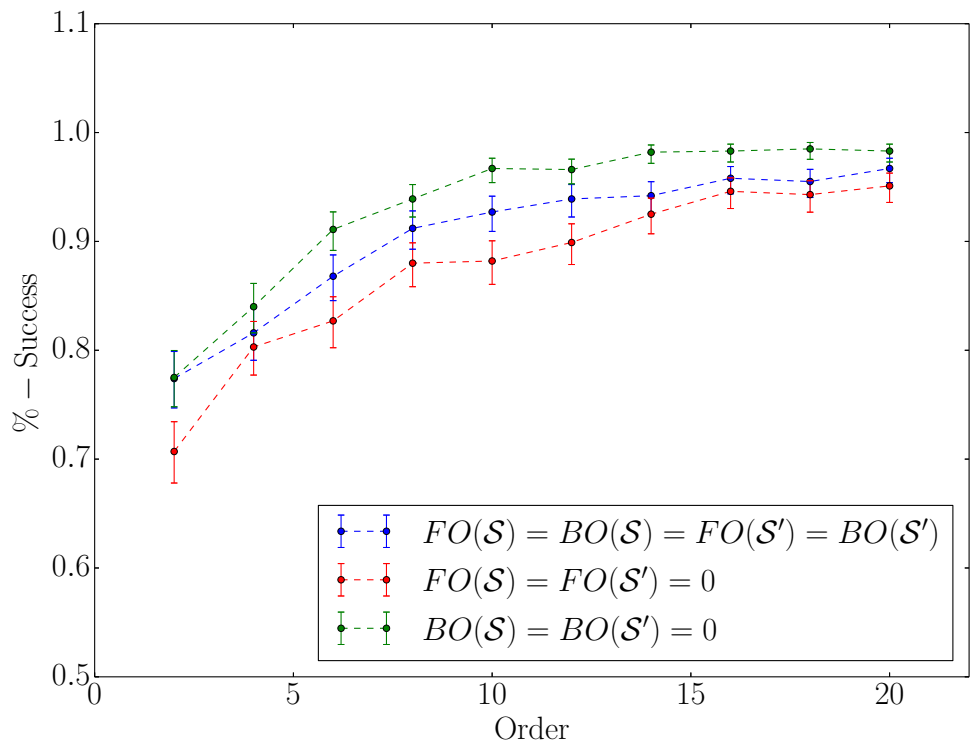}}\caption{SIC performance against filter order for synthetic experiments for different types of filters (see text).}
\label{fig:perf-dim-fo}
\end{figure}

\subsection{Real World Examples}
We tried our method over several examples of real data where the ground truth about the causal structure of the data is known a priori and the data is labeled in a way that the ground truth is $\textbf{X}\to\textbf{Y}$. In the first two examples we plotted the difference of SDR in both directions as a function of the window length used in Welch method which can be seen in ~ \cref{fig:gas_faithful}.
\subsubsection{Gas Furnace \cite{box2013time}}
This dataset consists in 296 time points, with $\textbf{X}$ the gas rate consumed by a gas furnace and $\textbf{Y}$ the produced rate of $\text{CO}_2$. ~\cref{fig:gas_faithful} shows $\rho_{\textbf{X}\to \textbf{Y}}-\rho_{\textbf{Y}\to \textbf{X}}$ against the window length, which was ranging from $50$ to $150$ points.  As illustrated, the difference is always positive and our method is able to correctly infer the right causal direction independent of window length. TiMiNO and Granger causality correctly identified the ground truth in this case as well \cite{peters2012causal}.

\subsubsection{Old Faithful Geyser \cite{azzalini1990look}}
$N=298$ : $\textbf{X}$ contains the duration of an eruption and \textbf{Y} is the time interval to the next eruption of the Old Faithful geyser. Figure \cref{fig:gas_faithful} represents the difference in SDRs as a function of window length with the same configuration as the gas furnace experiment. Again the correct causal direction is inferred by our method independently from the window length as illustrated in ~\cref{fig:gas_faithful}. In this case TiMiNO correctly identifies the cause from effect but neither linear nor non-linear Granger causality infer the correct causal direction \cite{peters2012causal}.
\begin{figure}[!hbt]
\centering
\includegraphics[width=\linewidth]{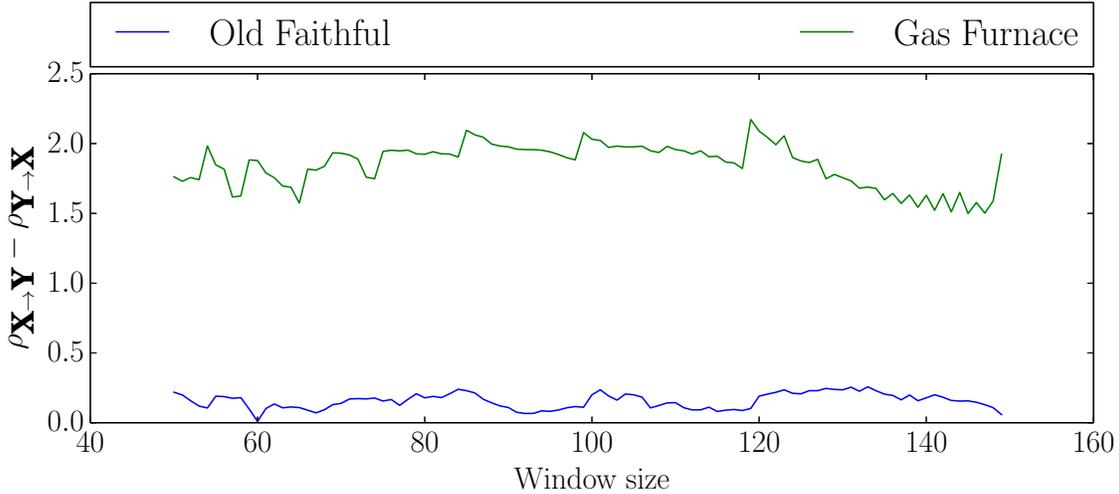}
\caption{Difference between the estimators of SDRs in both directions against window length of the Welch periodogram.}
\label{fig:gas_faithful}
\end{figure}
\subsubsection{LFP recordings of the Rat Hippocampus}
\label{ex:rat}
It is known that contrary to neocortex where connectivity between areas is bidirectional, monosynaptic connections between several regions of the hippocampus are mostly unidirectional \cite{andersen2006hippocampus}. An important example of such connectivity is between the CA3 and CA1 subfields \cite{andersen2006hippocampus}. Despite this anatomical fact, a study of causality based on Local Field Potential (LFP) recordings of CA1 and CA3 of the hippocampus of the rat during sleep reports  that Granger causality infers strong bidirectional relations between the two areas \cite{baccala1998studying}. \cite{baccala1998studying} explains the possible reasons of such result as feedback loops involving cortex and medial septum, and diffuse connections going from CA1 to CA3. 

To do a comparison with Granger causality, we applied our framework to recordings from those regions using a publicly available dataset\footnote{\href{http://crcns.org/data-sets/hc}{http://crcns.org/data-sets/hc}} \cite{mizuseki2009theta,epa:CRCNS}. LFP's were recorded using a $8$ shank probe having $64$ channels downsampled to $1252$Hz. Shanks were attributed by experimentalists to the CA1 and CA3 areas (leaving $32$ channels for each area). For more information on the details of gathered data please refer to \cite{epa:CRCNS}. We used the data for rat ``vvp01'' during a period of sleep and a period of active behaviour in a linear environment. 
We applied linear Granger causality using an implementation from the \textit{statsmodel} Python library\footnote{\href{http://statsmodels.sourceforge.net/}{Statsmodels: Statistical library for Python}. More details on null hypothesis for Granger causality can be found on the website.}. We considered a forced decision scheme for Granger causality (to make it comparable to our method), were we select the correct Granger causal direction as the one having the lowest $p$-value for the null hypothesis of absence of causal influence. Following the usual methodology of causality analysis  \cite{baccala1998studying,cadotte2010granger} we divided the duration of ten minutes into $300$ intervals of two seconds ($N=2504$) to reduce the effect of nonstationarity in data analysis, and performed SIC causal inference on each interval for each electrode pair. We took two different approaches to report assess the performance of methods: one, based on a majority vote over all $300$ intervals for each channel pair, and two, by assessing the average performance based on individual time intervals. The results are plotted as histograms in ~ \cref{fig:linear} and they show that SIC clearly outperforms Granger causality on this dataset. The confidence intervals are once again based on Wilson score but obviously this time the in dependancy assumption between the trials is not well justified, specially for pooling all the results.
\begin{figure}[!hbt]
\centerline{\includegraphics[width=\linewidth]{\detokenize{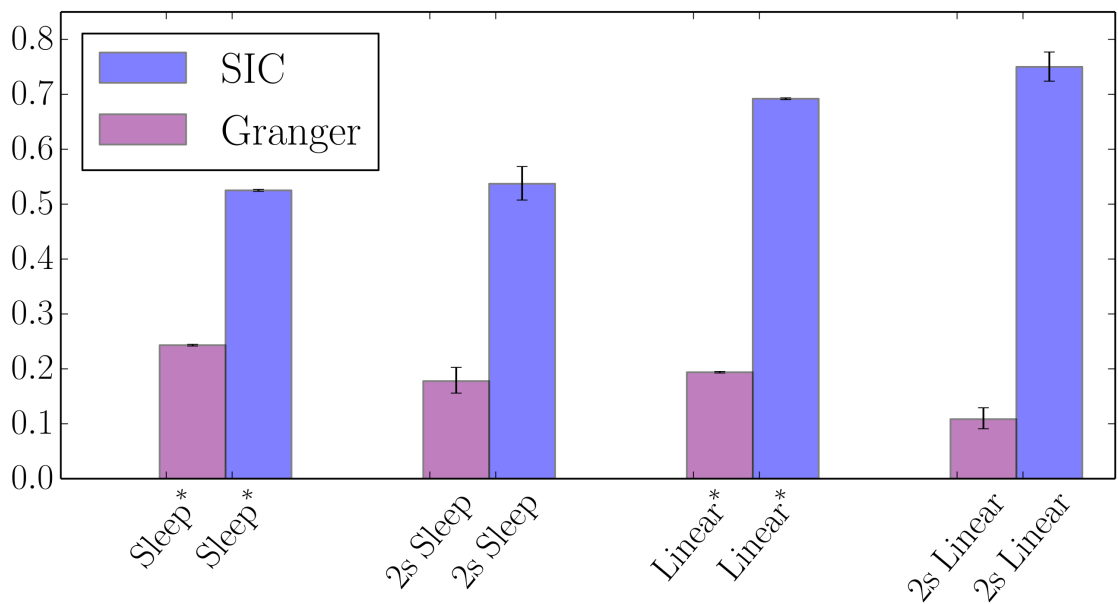}}}
\caption{Average performance of the linear Granger causality and SIC methods for deciding CA3$\to$CA1 ground truth direction against the opposite. For both the linear and sleep sessions the performance is significantly above the chance level for SIC. $*$ indicates the use of a majority voting scheme.}
\label{fig:linear}
\end{figure}
\subsubsection{Characterizing the Echo}
The echo effect of a room over a sound generated in the room can be well estimated by a convolution of the real signal with a function known as room impulse response function. In this experiment we used an open source database of room Impulse Response Function (IRF) available at the Open AIR library\footnote{\href{http://www.openairlib.net/}{{Open AIR}: Open source library for acoustic IRFs.}}. We chose the IRFs for Elevden Hall, Elevden, Suffolk, England and Hamilton Mausoleum, Hamilton, Scotland. We convolved these signals with $30\pm 5$ seconds segments of two classical music pieces: the first movement of Vivaldi's Winter Concerto consisting of $9190656$ data points, and the Lacrimosa of Mozart's Requiem, consisting of $8842752$ points, both `.wav' files with the rate of $44100$Hz. Regardless of the segment the SDR in forward direction is considerably larger than the SDR in the backward direction as can be seen in \cref{fig:echo}.
\begin{figure}[!hbt]
\centerline{\includegraphics[width=\linewidth]{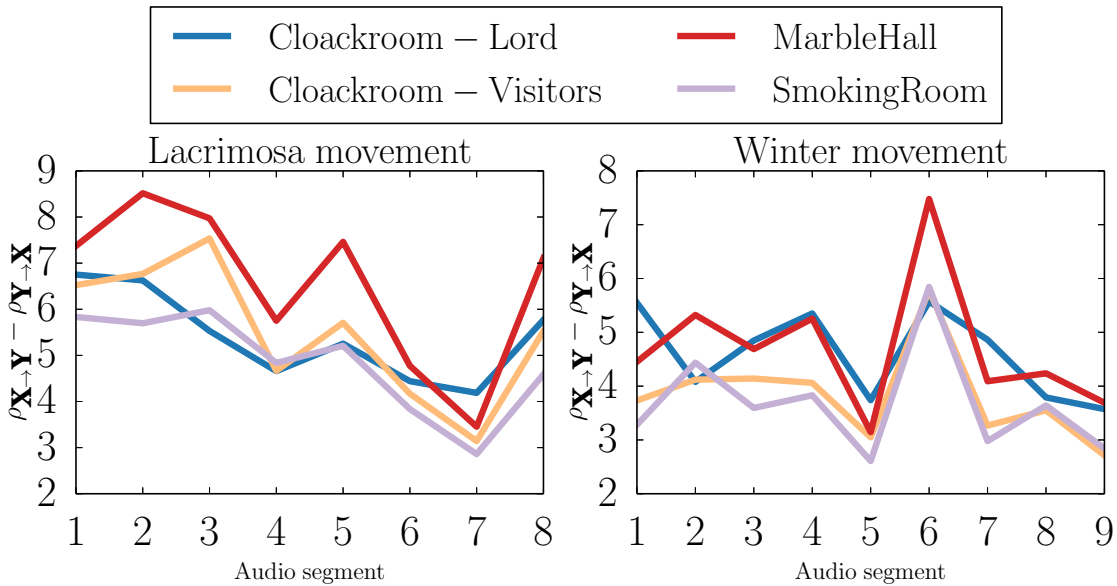}}
\caption{The plots represent the value $\rho_{X_t\to Y_t}-\rho_{Y_t\to X_t}$ for 4 different environments as a function of different music segments. The method correctly infers the causal direction in all the cases.}
\label{fig:echo}
\end{figure}
In another experiment we used a computer to play the musical pieces above in an academic Lecture Hall (labelled as ``Hall'' in plots) and in an office room (labelled as ``Room'' in plots) and recorded the echoed version in the environment. In a series of different tests, we split the data into $9,17,33,65,129$ pieces, and we ignored the last piece so that all the pieces would have an equal length. In each test we averaged the performance of our causal inference method over all the segments and plotted this performance against the size of the window length in Welch method. The window size was varied between $500$ and half of the length of the music segment length (which is dependent on the number of segments). The results can be found in ~\cref{fig:real_echo} and show a very good performance of the approach for large window lengths \michel{what is the length unit?}.
\begin{figure}[!hbt]
\includegraphics[width=\linewidth]{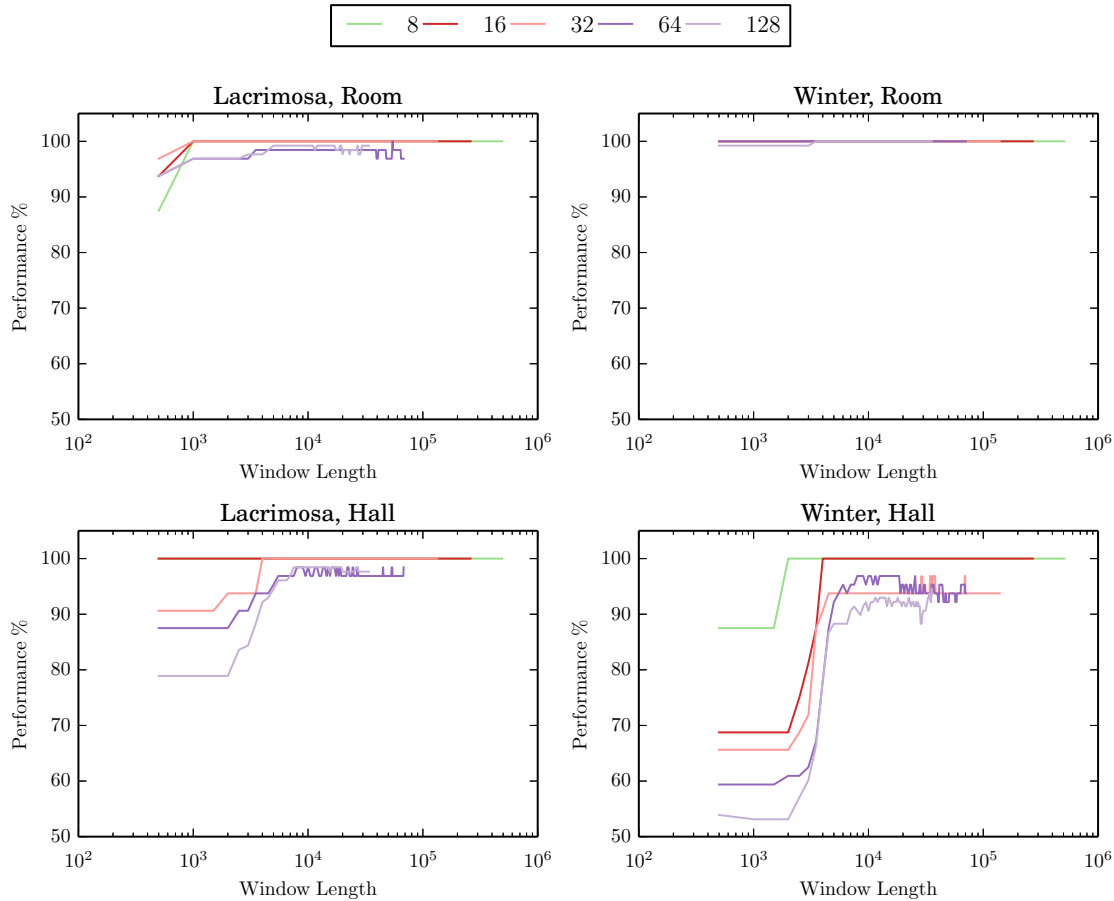}
\caption{The performance of the method over real echoed audio signals recorded simultaneously by playing the piece in two different closed environments that have their own acoustic structure.}
\label{fig:real_echo}
\end{figure}
\vspace*{-.5cm}
\section{Conclusion}
We have introduced a causal discovery method for time series based on the SIC postulate, assuming a LTI relationship for a given pair of time series $\textbf{X}$ and $\textbf{Y}$, such that either $\textbf{X} \to \textbf{Y}$ or $\textbf{Y} \to \textbf{X}$. Theoretical justifications are provided for this postulate to lead to identifiability. Interestingly, the method provides and extension of the recently proposed Trace Method approach to the time series setting. Encouraging experimental results have been also presented on real world and synthetic data. Specially this method proved to be more effective than linear Granger causality on LFP recordings from CA1 and CA3 hippocampal areas of rat's brain, assuming a ground truth causal direction from CA3 to CA1 based on anatomy. We suggest that this method can provide a new perspective for causal inference in time series based on assumptions fundamentally different from Granger causality. We will address the existence of confounders, establish a statistical significance test (for example using a procedure inspired by \cite{Zscheischler2011}), and extend this method to multivariate time series in future work.

\bibliographystyle{plain}
\bibliography{Bibliography}
\section*{Supplementary Material}

We have prepared an appendix to address the proofs for \Cref{lem:violation}, \cref{thm_COM_FIR,thm_limit} which we provide in two different sections. For this purpose we will use a few extra notations which we define here. We will also use $\tau(A)$ and $\tau_N(A)$ interchangeably for the normalized trace of a square matrix $A$ of order $N$.
\label{sec:fir}

\section{Proof of \Cref{lem:violation}}
\begin{lem}\label{lem_cs}
For $f\in L^2(\mathcal{I})$ positive, non-constant, such that $1/f\in L^2(\mathcal{I})$, we have
$$
\int_\mathcal{I} f(x)^2 dx . \int_\mathcal{I} \frac{1}{f(x)^2} dx > 1
$$
\end{lem}
\begin{proof}
Using Cauchy-Schwartz inequality for the scalar product \
$$
\langle f(x)\,,\,\frac{1}{f(x)}\rangle =\int_\mathcal{I} f(x). \frac{1}{f(x)} dx=1 \,.
$$ 
Inequality is strict since $f$ and $1/f$ are not collinear (otherwise $f$ would be constant).
\end{proof}

\begin{lem}\label{lem_nonconst}
Let $f\in L^1(\mathcal{I})$ be positive, non-constant, such that $1/f\in L^1(\mathcal{I})$ and $\int_\mathcal{I}f(x)dx =1$.

Assume $\,\exists \alpha>0, \forall x \in \mathcal{I},f(x)\leq 2-\alpha\,$, 

then
$$
\int_\mathcal{I}f(x)dx.\int_\mathcal{I}\frac{1}{f(x)}dx\geq 1+\alpha \int_\mathcal{I}(f(x)-1)^2 dx
$$
\end{lem}
\begin{proof}
We denote $s(x)=f(x)-1$. Then $\int_\mathcal{I}s(x)dx=0$ and
$$
\int_\mathcal{I}f(x)dx.\int_\mathcal{I}\frac{1}{f(x)}dx -1 = \int_\mathcal{I}\frac{-s(x)}{1+s(x)}dx
$$
For $x>-1$, we have 
\begin{gather}
\frac{-x}{1+x}\geq x^2 -x^3 -x.
\label{eq_conv}
\end{gather}
The function on the l.h.s. of (\cref{eq_conv}) is convex because its second order derivative $\frac{2}{(1+x)^3}$ is positive and using its tangent in $x=0$, we get
$$
\int_\mathcal{I}f(x)dx.\int_\mathcal{I}\frac{1}{f(x)}dx -1 \geq \int_\mathcal{I} s(x)^2 (1 -s(x))dx
$$
Since $1-s(x)=2-f(x)\geq \alpha >0$,
$$
\int_\mathcal{I}f(x)dx.\int_\mathcal{I}\frac{1}{f(x)}dx -1 \geq \alpha\int_\mathcal{I} s(x)^2 dx
$$
\end{proof}
\begin{proof}[Proof of \Cref{lem:violation}]
By using the definition of Spectral Dependency Ratios and \Cref{lem_cs} we get
$$
\rho_{\textbf{X}\to \textbf{Y}}\rho_{\textbf{Y}\to \textbf{X}}=\frac{1}{\langle |\hat{h}|^2 \rangle \langle 1/|\hat{h}|^2\rangle }<1
$$
Moreover, applying \Cref{lem_nonconst} to $f=|\hat{h}|^2/\int_\mathcal{I}|\hat{h}|^2=|\hat{h}|^2/\|\textbf{h}\|_2^2$ we get inequality~\cref{eq:sicviol2}.
\end{proof}

\section{Proof of \Cref{thm_COM_FIR}}
To prove this theorem we rely on a theorem from \cite{JanzingHS2010} and a corollary that we derive from it. 
\begin{thm} [concentration of measure for finite dimensional linear relationships]\cite{JanzingHS2010}\label{thm:CoM-orig}
Suppose $\Sigma$ is a given covariance matrix and suppose $A\in M_{n \times m}(\mathbb{R})$ is also a given matrix. Then if one generates $\Sigma_X=U\Sigma U^\top$ by uniformly choosing an orthogonal matrix $U$ from $O(n)$ then $\Sigma_X$ together with $A$, satisfies trace condition in probability when $n$ tends to infinity. More precisely for a given $\varepsilon$ there exist $\delta:=1-\exp(\kappa(n-1)\varepsilon^2)$, $\kappa$ being a constant where
\begin{gather*}
|\tau_m(A \Sigma_X A^\top)-\tau_n(\Sigma_X)\tau_m(AA^\top)|=
|\tau_m(AU\Sigma U^\top A^\top)-\tau_n(\Sigma)\tau_m(AA^\top)|\leq 2\varepsilon \|\Sigma\|\|AA^\top\|
\end{gather*}
holds with probability $\delta$.
\end{thm}
In the above theorem (and the rest of the document), $\|.\|$ applied to a matrix will refer to the operator norm. The following corollary is a direct consequence of the previous theorem:
\begin{cor}\label{cor:com-mech}
Suppose $\Sigma$ is a given covariance matrix and suppose $A\in M_{n \times m}(\mathbb{R})$ is also a given matrix. Then if one generates $A_U=AU$ by uniformly choosing an orthogonal matrix $U$ from $O(n)$ then $A_U$ together with $\Sigma$, satisfies trace condition in probability when $n$ tends to infinity More precisely for a given $\varepsilon$ there exist $\delta:=1-\exp(\kappa(n-1)\varepsilon^2)$, $\kappa$ being a constant where
\begin{gather*}
|\tau_m(A_U \Sigma A_U^\top)-\tau_n(\Sigma_X)\tau_m(AA^\top)|=\\
|\tau_m(AU\Sigma U^\top A^\top)-\tau_n(\Sigma)\tau_m(AA^\top)|\leq 2\varepsilon \|\Sigma\|\|AA^\top\|
\end{gather*}
holds with probability $\delta$.
\end{cor}

To prove the main theorem we will also need two lemmas that are stated below.
\begin{lem} \cite{serre2010matrices}
\label{cor:max-eig}
For a given Hermitian matrix $H$ and any principal submatrix of $H$, $H'$, their spectral radius $\rho_s$ satisfies
\begin{gather*}
\rho_s(H)\geq \rho_s(H').
\end{gather*}
\end{lem}

\begin{lem}\label{lem:szego}\cite{gray2006toeplitz}
Let $f:[-\frac{1}{2},\frac{1}{2})\to \mathbb{R}$ $f\in L^1$ be a bounded function and suppose $t_k$ is its Fourier series coefficients, i.e.
\begin{gather*}
t_k=\int_{-\frac{1}{2}}^{\frac{1}{2}}f(\nu)e^{i2\pi k\nu}d\nu,\ \ \ \  t\in \mathbb{Z}.
\end{gather*}
Consider Toeplitz matrices $T_n$ defined as 
$$[T_n]_{ij}=t_{i-j} \ \ \ \  i,j\in\{0,...,n-1\}$$
with eigenvalues $\tau_{n,k} (0\leq k\leq  n-1)$. Then if $t_i$ are absolutely summable we get:
\begin{gather*}
\min\limits_{x\in [-\frac{1}{2},\frac{1}{2})}f(x)\leq\tau_{n,i}\leq\max\limits_{x\in [-\frac{1}{2},\frac{1}{2})}f(x)
\end{gather*}
\end{lem}

\begin{proof}[Proof of \Cref{thm_COM_FIR}]
{
Without loss of generality and for the sake of simplicity we only consider the positive indices of the time series and we take the filter to be causal; other cases can be treated in a similar way. Then the following relation holds between input and output of the filter:
\begin{gather*}
\forall i,\ \ \ \   0\leq i\leq N-1\ \ \ \   Y_i=\sum\limits_{j=0}^{m-1}b_{j}X_{i-j}
\end{gather*}
Formulated in terms of matrices the above relation can be represented as
\begin{gather*}
\left[ {\begin{array}{c}
Y_{0}\\
Y_{1}\\
\vdots\\
Y_{N-2}\\
Y_{N-1}
\end{array} } \right]
=B
\left[ {\begin{array}{c}
X_{-m+1}\\
X_{-m+2}\\
\vdots\\
X_{N-2}\\
X_{N-1}
\end{array} } \right],
\end{gather*}
where $B$ is a $N\times (N+m-1)$ matrix as follows:
\begin{gather*}
\left[ {\begin{array}{cccccccc}
   b_{m-1}& b_{m-2} &  \cdots  & b_{0} & 0 &\cdots & 0& 0   \\
   0& b_{m-1} &  \cdots  & b_{1}& b_0 & \cdots & 0&0\\
   &  & \ddots &	 &\\
   0& 0 &  \cdots  & b_{m-1}&\cdots& b_1 & b_0 & 0\\
   0& 0 &  \cdots  & 0 &b_{m-1}&\cdots & b_1 & b_0\\
\end{array} } \right]
\end{gather*}
We define $\Sigma_X^i\in M_{m\times m}(\mathbb{R})$ to be the covariance matrices as follows:
\begin{gather*}
\forall i\ \ \ \   0\leq i\leq N-1\ \ \ \  0\leq j,k\leq m-1\ \ \ \  [\Sigma_X^i]_{jk}=\\
\cov(X_{i+j},X_{i+k})
\end{gather*}
Since the time series that we are dealing with are weakly stationary it is obvious that $\Sigma^i_X$ is independent of $i$. If we take $\Sigma_{X_{0:N-1}},\Sigma_{Y_{0:N-1}}\in M_{N\times N}(\mathbb{R})$ to be the covariance matrices for $X_{0:N-1}$ and $Y_{0:N-1}$ respectively, then we have
\begin{gather*}
\Sigma_{Y_{0:N-1}}=B\Sigma_{X_{-m+1:N-1}} B^\top
\end{gather*}
Also define $\Sigma_{Y_{0:N-1}}^U$ to be the covariance matrix of the output for FIR $\mathcal{S}'$ with $\textbf{b}'=U^\top\textbf{b}$. Also assume the spectrum of the output for this filter is $S_{yy}^U$. One can write diagonal elements of $\Sigma_{Y_{0:N-1}}$ and $\Sigma_{Y_{0:N-1}}^U$ based on the above equation as follows:
\begin{gather*}
[\Sigma_{Y_{0:N-1}}]_{ii}=\textbf{b}^\top\Sigma^i_{X}\textbf{b},\ \ \ \ 
[\Sigma_{Y_{0:N-1}}^U]_{ii}=\textbf{b}^\top U\Sigma^i_{X}U^\top \textbf{b}
\end{gather*}
and therefore the normalized traces of $\Sigma_{Y_{0:N-1}}$ and $\Sigma^U_{Y_{0:N-1}}$ can be written as
\begin{gather*}
\tau_N(\Sigma_{Y_{0:N-1}})=\frac{1}{N}\textbf{b}^\top\sum_{i=0}^{N-1}\Sigma^i_X\textbf{b},\\
\tau_N(\Sigma_{Y_{0:N-1}}^U)=
\frac{1}{N}\textbf{b}^\top U\sum_{i=0}^{N-1}\Sigma^i_X U^\top\textbf{b}
\end{gather*}
Define $\Sigma:=\sum_{i=0}^{N-1}\Sigma^i_X=\Sigma_X^0$. Taking $A=\textbf{b}^\top$ in corollary \cref{cor:com-mech} for a randomly selected $U$ we get
\begin{gather*}
|\frac{1}{N}\textbf{b}^\top U \Sigma U^\top\textbf{b} -\frac{1}{N}\tau_m(\Sigma)\langle\textbf{b},\textbf{b}\rangle|\leq 2\varepsilon \|\Sigma\|\sqrt{\langle\textbf{b},\textbf{b}\rangle}
\end{gather*}
and therefore
\begin{gather}\label{com_intermdiate}
|\tau_N(\Sigma_{Y_{0:N-1}}^U) -\frac{1}{N}\tau_m(\Sigma)\|\textbf{b}\|_2^2|\leq 2\varepsilon \|\Sigma\| \|\textbf{b}\|_2^2
\end{gather}
with probability $\delta$. On the other hand the elements of diagonals of $\Sigma^i_X$'s are $C_X(0)$. Therefore:
\begin{gather*}
\frac{1}{N}\tau_m(\Sigma)=\frac{mNC_X(0)}{mN}=P(\textbf{X})
\end{gather*}
Since $\Sigma^i_X$'s are principal submatrices of $\Sigma_{X_{0:N-1}}$ therefore by corollary \cref{cor:max-eig} 
\begin{gather*}
\|\Sigma\|=\rho(\Sigma)=\|\frac{1}{N}\sum_{i=0}^{N-1}\Sigma^i_X\|\leq
 \frac{1}{N}\sum_{i=0}^{N-1}\|\Sigma^i_X\| \leq \rho(\Sigma_{X_{0:N-1}}).
\end{gather*}
Because $C_X(\tau)$'s are absolutely summable we apply lemma \cref{lem:szego} and we get
\begin{gather*}
\rho(\Sigma_{X_{0:N-1}})\leq \max\limits_{\nu} S_{xx}(\nu),
\end{gather*}
such that inequality \cref{com_intermdiate} can be rewritten
\begin{gather*}
|\frac{\tau_N(\Sigma_{Y_{0:N-1}}^U)}{P(\textbf{X})\|\textbf{b}\|_2^2} -1|\leq 2\frac{\varepsilon}{P(X)} \|\Sigma\| 
\end{gather*}
which completes the proof.
}
\end{proof}

\section{Proof of \Cref{thm_limit}}

In this section we give a proof that the TDR  (see eq. \cref{eq:tdr}) asymptotically approaches the SDR (see eq. \cref{eq:sdr}). We first state and prove two lemmas that are used to derive this result. As before suppose $\{X_t\}$ and $\{Y_t\}$ are given input and output of an LTI filter that are related through the impulse response function $\{h_t\}$. According to the definition of the truncated linear systems (see definition \cref{def:trunc}) of order $N$ for the linear system above we get the following matrix relationship:
\begin{gather}\label{eq:windowed-eq}
\left[ {\begin{array}{c}
Y_{-N}'\\
Y_{-N+1}'\\
\vdots\\
Y_{N-2}'\\
Y_{N-1}'
\end{array} } \right]
=
\left[ {\begin{array}{cccc}
   h_{0}& h_{-1} &  \cdots  & h_{-2N+1}  \\
   h_{1}& h_{0} &  \cdots  & h_{-2N+2}  \\
   &	\vdots &\\
   h_{2N-2}& h_{2N-3} &  \cdots  & h_{-1}  \\
   h_{2N-1}& h_{2N-2} &  \cdots  & h_{0}  \\
\end{array} } \right]
\left[ {\begin{array}{c}
X_{-N}\\
X_{-N+1}\\
\vdots\\
X_{N-2}\\
X_{N-1}
\end{array} } \right].\notag
\end{gather}
If we name the vector on the left as $\mathbf{y}_N$, the matrix as $H^N$ and the right vector as $\mathbf{x}_N$ then the associated TDR yields:
\begin{gather}
r_{\mathbf{x}_N\to\mathbf{y}_N}=\frac{\tau_N(\Sigma_{\mathbf{y}_N})}{\tau_N(\Sigma_{\mathbf{x}_N}){\tau_{2N}(H^N{H^N}^T)}}
\end{gather}
Define $T_N:=\tau_{2N}(H^N{H^N}^\top).$ Now we show that $T_N$ converges to $ \|\textbf{h}\|_2^2$ 
the energy of the impulse response.
\begin{lem}
\label{lem:convseries}
Assume $\|\textbf{h}\|^2_2<+\infty$, then 
$$
\lim_{N\rightarrow+\infty} T_N=\|\textbf{h}\|^2_2
$$
\end{lem}
\begin{proof}
{
First lets simplify the expression for $T_N$:
\begin{gather}
T_N:=\tau_{2N}(H^N{H^N}^\top)=\frac{1}{2N}\sum\limits_{i,j} [H^N]_{ij}^2=
\sum\limits_{k=-2N+1}^{2N-1}|h_k|^2\frac{2N-|k|}{2N}=\notag\\
\sum\limits_{k=-2N+1}^{-1}|h_k|^2\frac{2N-|k|}{2N}+\sum\limits_{k=0}^{2N-1}|h_k|^2\frac{2N-|k|}{2N}\label{eq:partial-toeplitz}.
\end{gather}
It is easy to see that $T_N$ is an increasing sequence of $N$. Moreover it is bounded by 
$$\sum\limits_{-\infty}^{\infty}|h_k|^2<\infty.$$ 
Therefore this series converges. In order to show that it converges to $ \|\textbf{h}\|^2$, we first notice that for a given $\varepsilon$, there exist $m_0\in\mathbb{N}$ such that
\begin{gather}
\label{eq:tri}
\forall m>m_0\ \ \ \  |\sum\limits_{k=-m}^{m}|h_k|^2- \|\textbf{h}\|^2|<\varepsilon.
\end{gather}
Now take $N_{m_0}>\frac{m_0 2^{m_0+1}|h_{m_0}|^2}{\varepsilon}$. 
We have
\begin{gather*}
N_{m_0}>\frac{m_0 2^{m_0+1}|h_{m_0}|^2}{\varepsilon}\Rightarrow \frac{|h_{m_0}|^2 m_0}{2N_{m_0}}<\frac{\varepsilon}{2^{m_0+2}}.
\end{gather*}
Same can be done for any $0\leq k \leq m_0$, i.e. there exist $N_k$ such that:
\begin{gather*}
\frac{|h_{k}|^2 k}{2N_{k}}<\frac{\varepsilon}{2^{k+2}}
\end{gather*}
Now take $N_{\max}=\max\{N_0,N_1,...,N_{m_0}\}+1$. Then obviously we get:
\begin{gather*}
\left ||h_{k}|^2-\frac{|h_{k}|^2 (2N_{\max}-k)}{2N_{\max}}\right |<\frac{\varepsilon}{2^{k+2}}
\end{gather*}
And therefore:
\begin{gather}
\label{eq:tri2}
\sum\limits_{k=0}^{m_0} \left ||h_{k}|^2-\frac{|h_{k}|^2 (2N_{\max}-k)}{2N_{\max}}\right |<
\sum\limits_{k=0}^{m_0} \frac{\varepsilon}{2^{k+2}}<\frac{\varepsilon}{2}\,.
\end{gather}
}
Similar results hold for the first sum term in (\cref{eq:partial-toeplitz}) and by taking the maximum of two $N_{\max}$'s (say $N_{\max}'$) and considering the fact that $T_N$ is increasing and by the application of triangular inequality for (\cref{eq:tri}), we can easily infer that
$$\forall N>N_{\max}' \ \ \ \ \ \ \left |T_N- \|\textbf{h}\|^2\right|<\varepsilon.$$
\end{proof}

In order to get the main result, we also need to prove that $Y'_k$'s in (\cref{eq:windowed-eq}) are asymptotically converging to $Y_k$'s in the following sense:
\begin{lem}
\label{lem:windowed-cov}
Suppose an LTI filter $\mathcal{S}$ with zero mean weakly stationary processes as input ($\{X_t\}$) and output ($\{Y_t\}$) and impulse response function $\{h_t\}$ has been given. Then for the truncated linear systems we have:
\begin{gather*}
\lim_{N\to \infty} |\tau(\Sigma_{Y_{-N:N-1}})-\tau(\Sigma_{Y_{-N:N-1}'})|=0,
\end{gather*}
\end{lem}
\begin{proof}
{
For simplicity of calculations we name $2N$ dimensional random vectors $Y_{-N:N-1}'$ and $Y_{-N:N-1}$ as $Y'$ and $Y$ and their covariance matrices with $\Sigma_{Y'}$ and $\Sigma_Y$ respectively. Then we have:
\begin{gather*}
\left |\tau(\Sigma_{Y_{-N:N-1}})-\tau(\Sigma_{Y_{-N:N-1}'})\right |
=\left |\tau(\E(YY^\top))-\tau(\E(Y'{Y'}^\top))\right|\overset{*}{=}
\frac{1}{2N}\left |\E(Y^\top Y)-\E(Y'^\top Y')\right |=\\
\frac{1}{2N}\left |\E\bigl((Y-Y')^\top(Y+Y')\bigr)\right |\leq
\frac{1}{2N}\E\left |(Y-Y')^\top(Y+Y')\right |\leq\\
\frac{1}{2N}\E\Bigl(\sqrt{(Y-Y')^\top (Y-Y')}\times
\sqrt{(Y+Y')^\top(Y+Y')}\Bigr)\leq\\
\frac{1}{2N}\sqrt{\E((Y-Y')^\top (Y-Y'))}\times
\sqrt{\E((Y+Y')^\top(Y+Y'))}=\\
\sqrt{\frac{1}{2N}\E((Y-Y')^\top (Y-Y'))}\times
\sqrt{\frac{1}{2N}\E((Y+Y')^\top(Y+Y'))}\overset{**}{=}
\sqrt{\tau(\Sigma_{Y-Y'})}\sqrt{\tau(\Sigma_{Y+Y'})}
\end{gather*}
where (*) and (**) follows from the fact that one can take trace (or normalized trace) into expectation and vice versa, and moreover from the fact that $\text{tr}(AB)=\text{tr}(BA)$ for any two matrices that their multiplication is well defined. The inequalities are the result of the application of Cauchy-Schwartz inequality for covariances of random variables. First we show that $\sqrt{\tau(\Sigma_{Y+Y'})}$ is bounded as a function of $N$. Define $\{h_t^{(j)}\}$ as follows
\begin{gather*}
h^{(j)}_t=
\begin{cases}
2 h_t & \text{if\ } -N\leq t+j\leq N-1\\
h_t & \text{otherwise}
\end{cases}.
\end{gather*}
We can bound each element of diagonal of $\Sigma_{Y+Y'}$ as follows
\begin{gather*}
[\Sigma_{Y+Y'}]_{jj}=\E\Bigl[(Y_j+Y_j')^2\Bigr]=\E\Bigl[(\sum_{l=-\infty}^{\infty}X_{j-l}h_l^{(j)})^2\Bigr]\leq
\E\Bigl[(\sum_{l=-\infty}^{\infty}|X_{j-l}||h_l^{(j)}|)^2\Bigr]\leq\\
4\E\Bigl[(\sum_{l=-\infty}^{\infty}|X_{j-l}||h_l|)^2\Bigr]=4C_{Y}(0),
\end{gather*}
and therefore $\tau(\Sigma_{Y+Y'})$ is bounded. \\

Now we show that each element of diagonal of $\Sigma_{Y-Y'}$ tends to zero when $N$ tends to infinity which will complete the proof. With overload of notation, in this case define $\{h_t^{(j)}\}$ as follows
\begin{gather*}
h^{(j)}_t=
\begin{cases}
0 & \text{if\ } -N\leq t+j\leq N-1\\
h_t & \text{otherwise.}
\end{cases}
\end{gather*}
Then for the $j$-th element of diagonal of $\Sigma_{Y-Y'}$ we have
\begin{gather*}
[\Sigma_{Y-Y'}]_{jj}=\E\Bigl[(Y_j-Y_j')^2\Bigr]=
\E\Bigl[(\sum_{l=-\infty}^{\infty}X_{j-l}h_l^{(j)})^2\Bigr]=\E\Bigl[(\sum_{\substack{l\geq N-j \\l< -N-j}}X_{j-l}h_l)^2\Bigr]
\end{gather*}
Since autocorrelation function attains its maximum at $t=0$ and 
\begin{gather*}
\forall i,j\in\mathbb{Z},\ \ \ \  \E(X_i X_j)\leq\sqrt{\E(X_i^2)\E(X_j^2)}
\end{gather*}
we get:
\begin{gather*}
\forall i,j\in\mathbb{Z},\ \ \ \  \E(X_i X_j)\leq \E(X_0^2).
\end{gather*}
As a result we have:
\begin{gather*}
[\Sigma_{Y-Y'}]_{jj}=\E\Bigl[(\sum_{\substack{l\geq N-j \\l< -N-j}}X_{j-l}h_l)^2\Bigr]\leq
\sum_{\substack{l,l'\geq N-j \\l,l'< -N-j}}\E(X_0^2)h_lh_{l'}=
\E(X_0^2)\sum_{\substack{l,l'\geq N-j \\l,l'< -N-j}}h_l h_{l'}\leq\\
\E(X_0^2)(\sum_{\substack{l\geq N-j \\l< -N-j}}h_l)^2\leq\E(X_0^2)(\sum_{\substack{l\geq N-j \\l< -N-j}}|h_l|)^2
\end{gather*}
Now since $\{h_t\}$ is absolutely convergent, it follows that $[\Sigma_{Y-Y'}]_{jj}$ can be arbitrarily reduced by increasing $N$. Then it follows that $\tau(\Sigma_{Y-Y'})$ approaches to zero when $N$ tends to infinity.
}
\end{proof}
Finally to complete the proof of the theorem regarding the asymptotic behaviour of trace condition in the truncated linear systems and the equivalence of trace condition (see postulate \cref{pos:tc}) to SIC, we need one of the convergence theorems due to Szeg\"o:
\begin{thm}[Szeg\"o's convergence theorem]\label{thm:szego}\cite{gray2006toeplitz} Let $f:[-\frac{1}{2},\frac{1}{2})\to \mathbb{R}$ $f\in L^1$ be a bounded function and suppose $t_k$'s are its Fourier series coefficients, i.e.
\begin{gather*}
t_k=\int_{-\frac{1}{2}}^{\frac{1}{2}}f(\nu)e^{\textbf{i} 2\pi k\nu}d\nu,\ \ \ \  t\in \mathbb{Z}.
\end{gather*}
Consider Toeplitz matrices $T_n$ defined as 
$$[T_n]_{ij}=t_{i-j} \ \ \ \  i,j\in\{0,...,n-1\}$$
with eigenvalues $\tau_{n,k} (0\leq k\leq  n-1)$. Then if $T_n$'s are Hermitian, i.e. $t_i=\bar{t_i}$ for any $i$, then for any continuous function $F$ we have:
\begin{gather*}
\lim\limits_{n\to \infty} \frac{1}{n}\sum\limits_{k=0}^{n-1}F(\tau_{n,k})=\int_{-\frac{1}{2}}^{\frac{1}{2}}F(f(\nu))d\nu
\end{gather*}
\end{thm}
We are ready to state our convergence theorem:
\begin{thm}\label{thm:limit}
For a given truncated linear time series,  $r_{\textbf{X}'_N\to\textbf{Y}'_N}$ asymptotically approaches to the spectral values of time series on infinite domain. As a result the spectral density based estimator coincides with the trace based estimator in the limit, and more precisely
\begin{gather*}
\lim\limits_{N\to \infty} \tau(\Sigma_{\mathbf{x_N}})=\int\limits_{-\frac{1}{2}}^{\frac{1}{2}}S_{xx}(\nu)d\nu,\indent
\lim\limits_{N\to \infty} \tau(\Sigma_{\mathbf{y_N}})=\int\limits_{-\frac{1}{2}}^{\frac{1}{2}}S_{yy}(\nu)d\nu,\\
{\rm and}\ \ \ \ \ \lim\limits_{N\to \infty} T_N=\int\limits_{-\frac{1}{2}}^{\frac{1}{2}}|\hat{h}(\nu)|^2d\nu,
\end{gather*}
where $T_N$ is defined as in (\cref{eq:partial-toeplitz}). And eventually:
\begin{gather*}
\lim\limits_{n\to \infty} r_{\textbf{X}'_N\to\textbf{Y}'_N}=\rho_{\textup{\textbf{X}}\to\textup{\textbf{Y}}}\ \ \ \ 
\lim\limits_{n\to \infty} r_{\textbf{Y}'_N\to\textbf{X}'_N}=\rho_{\textup{\textbf{Y}}\to\textup{\textbf{X}}}
\end{gather*}
\end{thm}
\begin{proof}
{
Both $\Sigma_{\bf{x}_N}$ and $\Sigma_{\bf{y}_N}$ are hermitian Toeplitz matrices and based on  theorem \cref{thm:szego} where $F$ has been chosen as identity function and also applying lemma \cref{lem:windowed-cov} we get:
\begin{gather}
\lim\limits_{N\to \infty} \tau(\Sigma_{\mathbf{x_N}})=\int\limits_{-\frac{1}{2}}^{\frac{1}{2}}S_{xx}(\nu)d\nu\\
\lim\limits_{N\to \infty} \tau(\Sigma_{\mathbf{y_N}})=\int\limits_{-\frac{1}{2}}^{\frac{1}{2}}S_{yy}(\nu)d\nu
\end{gather}
Moreover by Plancherel's theorem and lemma \cref{lem:convseries} it follows that:
\begin{gather}
\lim\limits_{N\to \infty} T_N= \|\textbf{h}\|_2^2=\int\limits_{-\frac{1}{2}}^{\frac{1}{2}}|\hat{h}(\nu)|^2d\nu
\end{gather}
}
\end{proof}
This theorem therefore shows that the trace ratios calculated for windowed version of time series are nothing but estimates of the spectral ratios and therefore justifies that these two different methods for causal inference are indeed consistent with each other.

\end{document}